\documentclass[a4paper, 11pt]{article}

\usepackage{microtype}
\usepackage{graphicx}
\usepackage{subfigure}
\usepackage{booktabs}

\usepackage[hidelinks]{hyperref}

\usepackage{amsmath}
\usepackage{amssymb}
\usepackage{mathtools}
\usepackage{amsthm}

\usepackage[capitalize,noabbrev,nameinlink]{cleveref}

\theoremstyle{plain}
\newtheorem{theorem}{Theorem}[section]

\newtheorem{lemma}[theorem]{Lemma}

\theoremstyle{definition}
\newtheorem{definition}[theorem]{Definition}

\theoremstyle{remark}
\newtheorem{remark}[theorem]{Remark}

\usepackage{fullpage}
\usepackage{centernot}
\usepackage{bm}
\usepackage{bbm}
\usepackage{thm-restate}

\usepackage{tikz}
\usetikzlibrary{shapes.geometric, shapes.arrows}
\usetikzlibrary{positioning}
\usetikzlibrary{decorations.pathmorphing}
\usetikzlibrary{calc}
\usetikzlibrary{fit}

\newcommand{\cprn}[1]{\!\left(#1\right)}

\newcommand{\abs}[1]{\left|#1\right|}

\newcommand{\bars}[1]{\left\|#1\right\|}
\newcommand{\eps}{\varepsilon}
\newcommand{\N}{\mathbb{N}}
\newcommand{\Pa}{\mathbf{Pa}}
\newcommand{\pa}{\mathbf{pa}}

\newcommand{\ind}{\perp\!\!\!\perp}
\newcommand{\NP}{\mathsf{NP}}
\newcommand{\poly}{\mathrm{poly}}
\newcommand{\tv}{d_\mathrm{TV}}
\newcommand{\dtv}{\tv}
\newcommand{\PROMISE}{\texttt{REALIZABLE}-\texttt{LEARN}}
\newcommand{\LEARN}{\texttt{LEARN}}
\newcommand{\DBFAS}{\texttt{DBFAS}}
\newcommand{\LEARNDBFAS}{\texttt{LEARN}-\texttt{DBFAS}}

\newcommand{\bE}{\bm{E}}
\newcommand{\bS}{\bm{S}}
\newcommand{\bV}{\bm{V}}
\newcommand{\bX}{\bm{X}}

\newcommand{\bZ}{\bm{Z}}

\newcommand{\bd}{\bm{d}}
\newcommand{\bx}{\bm{x}}

\newcommand{\bbP}{\mathbbm{P}}
\newcommand{\bbQ}{\mathbbm{Q}}

\newcommand{\Qset}{\pmb{\mathbb{Q}}}

\newcommand{\cG}{\mathcal{G}}
\newcommand{\cH}{\mathcal{H}}
\newcommand{\cK}{\mathcal{K}}

\newcommand{\cO}{\mathcal{O}}

\allowdisplaybreaks

\title{Learnability of Parameter-Bounded Bayes Nets}

\author{
{\bf Arnab Bhattacharyya} \\
National University of Singapore
\and
{\bf Davin Choo} \\
National University of Singapore
\and
{\bf Sutanu Gayen} \\
Indian Institute of Technology Kanpur
\and
{\bf Dimitrios Myrisiotis} \\
CNRS@CREATE LTD.
}

\begin{document}

\maketitle

\begin{abstract}
Bayes nets are extensively used in practice to efficiently represent joint probability distributions over a set of random variables and capture dependency relations.
In a seminal paper, Chickering et al. (JMLR 2004) showed that given a distribution $\bbP$, that is defined as the marginal distribution of a Bayes net, it is $\NP$-hard to decide whether there is a parameter-bounded Bayes net that represents $\bbP$.
They called this problem \LEARN.
In this work, we extend the $\NP$-hardness result of \LEARN\ and prove the $\NP$-hardness of a promise search variant of \LEARN, whereby the Bayes net in question is guaranteed to exist and one is asked to find such a Bayes net.
We complement our hardness result with a positive result about the sample complexity that is sufficient to recover a parameter-bounded Bayes net that is close (in TV distance) to a given distribution $\bbP$, that is represented by some parameter-bounded Bayes net, generalizing a degree-bounded sample complexity result of Brustle et al. (EC 2020).
\end{abstract}

\section{Introduction}

\label{sec:intro}

Bayesian networks \cite{pearl1988probabilistic}, or simply Bayes nets, are directed acyclic graphs (DAGs), accompanied by a collection of conditional probability distributions (that is, one for each vertex), that are used to represent joint probability distributions over dependent random variables in an elegant and succinct manner.
As an example, consider a distribution $\bbP$ over five Boolean variables $\bX = \{X_1, X_2, X_3, X_4, X_5\}$.
Regardless of the dependencies between the variables, $\bbP$ can always be represented by a lookup table that has $2^5 - 1 = 31$ entries, with one entry for each possible Boolean outcome except one.
However, known dependencies between variables may induce a sparser representation.
If $X_2$ depends on $X_1$, $X_3$ depends on $\{X_2, X_5\}$, and $X_4$ depends on $X_3$, then the joint distribution $\bbP(x_1, \ldots, x_5)$ decomposes as
\[
\bbP(x_1) \cdot \bbP(x_2 \mid x_1) \cdot \bbP(x_3 \mid x_2, x_5) \cdot \bbP(x_4 \mid x_3) \cdot \bbP(x_5).
\]
In fact, one can represent $\bbP$ with a relatively sparse Bayes net $\cG$ (see \cref{fig:Bayes-net}) with conditional probability tables (CPTs) associated with each vertex.
Observe that $8 < 31$ numbers suffice to describe the CPTs:
One for each of the Bernoulli distributions of $X_1$ and $X_5$, two for the conditional probability distributions of $X_2$ and $X_4$, and four for that of $X_3$.
In the rest of this work, we refer to the numbers used in defining the CPTs above as \emph{parameters}.

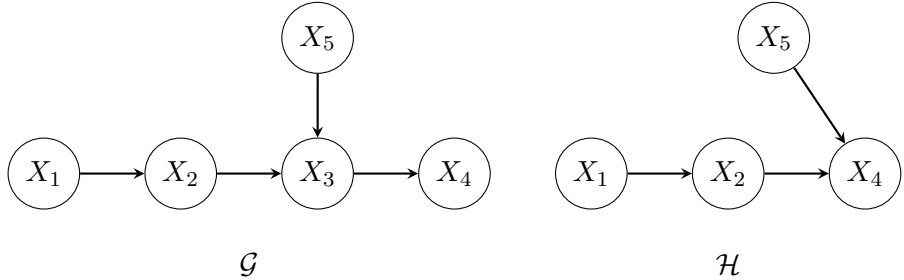
\begin{figure}[ht]
\centering
\begin{tikzpicture}[scale=1.2]
\node[draw, circle] (a1) at (0,0) {$X_1$};
\node[draw, circle] (a2) at (1.5,0) {$X_2$};
\node[draw, circle] (a3) at (3,0) {$X_3$};
\node[draw, circle] (a4) at (4.5,0) {$X_4$};
\node[draw, circle] (a5) at (3,1.5) {$X_5$};
\draw[thick, -stealth] (a1) -- (a2);
\draw[thick, -stealth] (a2) -- (a3);
\draw[thick, -stealth] (a3) -- (a4);
\draw[thick, -stealth] (a5) -- (a3);
\node[] at (2.25,-1) {$\cG$};

\node[draw, circle] (b1) at (6,0) {$X_1$};
\node[draw, circle] (b2) at (7.5,0) {$X_2$};
\node[draw, circle] (b4) at (9,0) {$X_4$};
\node[draw, circle] (b5) at (8,1.5) {$X_5$};
\draw[thick, -stealth] (b1) -- (b2);
\draw[thick, -stealth] (b2) -- (b4);
\draw[thick, -stealth] (b5) -- (b4);
\node[] at (7.5,-1) {$\cH$};
\end{tikzpicture}
\caption{\emph{Left:} A Bayes net $\cG$ such that the distribution $\bbP$ of our example is represented by $\cG$.
\emph{Right:} A Bayes net $\cH$ such that the distribution that arises from the distribution $\bbP$ after marginalizing out $X_3$ is represented by $\cH$.}
\label{fig:Bayes-net}
\end{figure}

It is a standard result~\cite{pearl1988probabilistic,chickering2004large} that there exists a Bayes net $\cG$ of $p$ parameters that represents a probability distribution $\bbP$ if and only if $\bbP$ is Markov with respect to some Bayes net $\cH$ of $p$ parameters that has the same underlying DAG as $\cG$.
(It could be that $\cG = \cH$, but this is not necessary.)
Here, the property of a distribution $\bbP$ being Markov with respect to a Bayes net $\cG$ means that a certain graphical separation condition in the underlying DAG of $\cG$, known as d-separation, implies conditional independence in $\bbP$ (see \cref{sec:prel-graphs} for formal definitions).
We shall make use of this equivalence in the sequel.

A series of works studied the problem of learning the underlying DAG of a Bayes net from data, by focusing on maximizing certain scoring criterion by the underlying DAG, see, e.g., \cite{cooper1992bayesian,spiegelhalter1993bayesian,heckerman1995learning}.
This task was later shown to be $\NP$-hard by \cite{chickering1996learning}, which then raised the following natural fundamental question:
\begin{quote}
Given a succinct description of a distribution $\bbP$ (that is \emph{not} in terms of a Bayes net), how easy is it to find a Bayes net $\cG$ such that $\bbP$ is Markov with respect to $\cG$?
\end{quote}
Unfortunately, \cite{chickering2004large} showed that deciding whether a given distribution $\bbP$ is Markov with respect to some Bayes net of at most $p \in \N$ parameters or not is $\NP$-hard.

\begin{remark}
In \cite{chickering2004large}, the distribution $\bbP$ is described as the certain \emph{marginal of a Bayes net} of small in-degree, i.e., a succinct Bayes net description over variables $\bX$, along with a subset of variables $\bS \subseteq \bX$ to marginalize out.
For example, any distribution $\bbP'$ that is Markov with respect to the left Bayes net in \cref{fig:Bayes-net} is Markov with respect to the right Bayes net in \cref{fig:Bayes-net} after marginalizing out $X_3$.
Note that all possible distributions over $\bX$ are Markov with respect to some Bayes net over a clique, but such a Bayes net requires $2^{|\bX|} - 1$ parameters.
\end{remark}

Regarding upper bounds, there are well-known algorithms for learning the underlying DAG of a Bayes net from distributional samples such as the PC \cite{spirtes2000causation} and GES \cite{chickering2002optimal} algorithms.
Recently, \cite{brustle2020multi} also gave finite sample guarantees of learning Bayes nets (that have $n$ nodes, each taking values over an alphabet $\Sigma$) from samples.
When given the promise that the underlying DAG has bounded in-degree of $d$, \cite[Theorem 10]{brustle2020multi} asserts that using
\begin{equation}
\label{eq:BCD20-bound}
\cO\left( \frac{\log\frac{1}{\delta}}{\eps^2} \left( n \abs{\Sigma}^{d+1} \log \left( \frac{n {\color{black}\abs{\Sigma}}}{\eps} \right) + n \cdot d \cdot \log n \right) \right)
\end{equation}
samples from the underlying distribution $\bbP$, one can learn $\bbP$ up to total variation (TV) distance $\eps$ with probability at least $1-\delta$.

One standard way to measure the complexity of a Bayes net is by imposing an upper bound on any node's in-degree.
In this work, we measure the complexity in terms of the number of parameters $p$, which is a more fine-grained measure than that of maximum in-degree $d$.
For instance, a Bayes net on $n$ Boolean variables with maximum in-degree $d$ could have as few as $O(n + 2^d)$ parameters (e.g., a star graph where $d$ leaves point towards the center of the start) and as many as $\Omega(n \cdot 2^d)$ parameters (e.g., a complete graph).

\subsection{Our Contributions}

Our two main contributions are that we extend the hardness result of \cite{chickering2004large} and generalize the finite sample complexity result of \cite{brustle2020multi}.

\paragraph{Contribution 1.}

We extend the hardness result of \cite{chickering2004large} to the setting where we are guaranteed that the Bayes net in question is promised to have a small number of parameters.
In computational complexity theory, this is also known as a \emph{promise problem}, which generalizes a decision problem in that the input is promised to belong to a certain subset of possible inputs.
Our new hardness result confirms the common intuition that it is hard to search for a Bayes net $\cG$ that is Markov with respect to a given probability distribution, even if it is known that the distribution in question is Markov with respect to a Bayes net that has a small number of parameters.

\begin{definition}[The \PROMISE\ problem]
Given as input variables $\bX = (X_1, \ldots, X_n)$, a probability distribution $\bbP$ on $\bX$, a parameter bound $p \in \N$, and the promise that there exists a Bayes net $\cG$ with at most $p$ parameters such that $\bbP$ is Markov with respect to $\cG$, output a Bayes net $\cG'$ with at most $p$ parameters such that $\bbP$ is Markov with respect to $\cG'$.
\end{definition}

\begin{restatable}{theorem}{nphardthm}
\label{thm:REALIZABLE-LEARN-is-NP-hard}
\PROMISE\ is $\NP$-hard.
\end{restatable}

Technically speaking, \PROMISE\ is a promise search problem.
While $\NP$-hardness results usually revolve around decision problems, $\NP$-hardness naturally extends to the more general case of search problems when Turing reductions are considered.
(Turing reductions comprise a very broad class of reductions, whereby an efficient algorithm for a problem yields an efficient algorithm for another.)

\paragraph{Contribution 2.}

We generalized the finite sample result of \cite{brustle2020multi} from the \emph{degree-bounded} setting to the \emph{parameter-bounded} setting.

\begin{restatable}[Approximating parameter-bounded Bayes nets using samples]{theorem}{samplecomplexitythm}
\label{thm:parameter-sample-complexity}
Fix any accuracy parameter $\eps > 0$ and confidence parameter $\delta > 0$.
Given sample access to a distribution $\bbP$ over $n$ variables, each defined on the alphabet $\Sigma$, and the promise that $\bbP$ is Markov with respect to a Bayes net with at most $p$ parameters,
\[
\cO \! \left( \frac{\log \frac{1}{\delta}}{\eps^2} \left( p \log \left( \frac{n \abs{\Sigma}}{\eps} \right) + n \frac{\log \left( \frac{p}{n \left( |\Sigma| -1 \right)} \right)}{\log \abs{\Sigma}} \log n \right) \right)
\]
samples from $\bbP$ suffice to learn the underlying DAG of a Bayes net $\cG$ with at most $p$ parameters and define a distribution $\bbQ$ that is Markov with respect to $\cG$ such that $\tv(\bbP, \bbQ) \leq \eps$ with success probability at least $1 - \delta$.
\end{restatable}

Notice that when all in-degrees are at most $d$, we have $p \leq (\abs{\Sigma} - 1) \cdot n \cdot \abs{\Sigma}^d$, so our result generalizes the bound of~\cite{brustle2020multi} given in \cref{eq:BCD20-bound}.\footnote{One can further generalize \cref{thm:parameter-sample-complexity} to the case where each node has a different alphabet size, e.g., $X_i$ has alphabet $\Sigma_i$, but this is a straightforward extension.}

Finally, we note that while \cref{thm:parameter-sample-complexity} runs in time polynomial in $1/\delta$ and $1/\eps^2$, and it has exponential dependency on the number of samples from $\bbP$, similar to \cite{brustle2020multi}.

\subsection{Paper Outline}

After preliminaries in \cref{sec:preliminaries}, we give a high-level overview of the techniques behind our results in \Cref{sec:technical-overview}.
We then formally prove \cref{thm:REALIZABLE-LEARN-is-NP-hard} and \cref{thm:parameter-sample-complexity} in \cref{sec:np-hard} and \cref{sec:upper-bound}, respectively.
Finally, we conclude with an open problem in \cref{sec:discussion}.

\section{Preliminaries}

\label{sec:preliminaries}

\subsection{Notation}

We define the set of natural numbers by $\N$ and all logarithms refer to the natural $\log$.
Distributions are written as $\bbP$, $\bbQ$ and graphs in calligraphic letters, e.g., $\mathcal{G}, \mathcal{H}, \mathcal{K}$.
For variables/nodes, we use capital letters, small letters for the values taken by them, and boldface versions for a collection of variables, e.g., $X = x$ and $\bX = \bx$.
As shorthands, we write $[n]$ for $\{1, \ldots, n\}$ and $\bbP(\bx)$ for $\bbP(\bX = \bx)$.
We will often represent the same set of variables of distributions as nodes in a graph.

Problems and algorithms are named in the typewriter font in full caps and capitalized, respectively, e.g., $\texttt{PROBLEM}$ and $\texttt{Algorithm}$.

We will also often use $\Sigma$ to denote the alphabet set of a variable and write $\Delta_{|\Sigma|}$ to denote the corresponding (conditional) probability simplex.

\subsection{Graph-Theoretical Notions}

\label{sec:prel-graphs}

Let $\cG = (\bX, \bE)$ be a fully directed graph on $|\bX| = n$ vertices and $|\bE|$ edges where adjacencies are denoted with dashes, e.g., $X - Y$, and arc directions are denoted with arrows, e.g., $X \to Y$.
For any node $X \in \bX$, we write $\Pa_{\cG}(X) \subseteq \bX$ to denote its parents and $\pa_{\cG}(X)$ to denote the values they take.

A \emph{degree sequence of a graph $\cG$} on vertex set $\bX = \{X_1, \ldots, X_n\}$ is a list of degrees $\bd = (d_1, \ldots, d_n)$ of all vertices in the graph, i.e., vertex $X_i$ has degree $d_i$.
A graph $\cH = (\bX, \bE)$ is said to \emph{realize} degree sequence $\bd$ if the degrees of $\bX$ in $\cH$ agree with $\bd$.
Realizability is defined in a similar fashion for in-degree sequences $\bd^- = (d_1^-, \ldots, d_n^-)$.

The graph $\cG$ is called a \emph{directed acyclic graph (DAG)} if it does not contain any directed cycles and is said to be \emph{complete} if for every two of its nodes $U,V \in \bX$ either there is an edge $V \to U$ or an edge $U\to V$, i.e., the underlying undirected graph is a clique.
A vertex $V_i$ on any simple path $V_1 - \ldots - V_k$ is called a \emph{collider} if the arcs are such that $V_{i-1} \to V_i \gets V_{i+1}$.

A Bayesian network (or Bayes net) $\cG$ for a set of $n$ variables $X_1, \ldots, X_n$ is described by a DAG $(\bX, \bE)$ and $n$ corresponding conditional probability tables (CPTs), e.g., the CPT for $X_i \in \bX$ describes $\bbP(x_i \mid \pa_{\cG}(X_i))$ for all possible values of $x_i$ and $\pa_{\cG}(X_i)$.
The joint distribution for $\bbP$ factorizes as
\[
\bbP(\bx) = \prod_{i=1}^n \bbP(x_i \mid \pa_{\cG}(X_i)),
\]
and we say that \emph{$\cG$ represents $\bbP$}.

All independence constraints that hold in the joint distribution of a Bayes net that has underlying DAG $\cG$ are exactly captured by the \emph{d-separation} criterion \cite[Section 3.3.1]{pearl1988probabilistic}.
Two nodes $X, Y \in \bX$ are said to be \emph{d-separated} in a DAG $\cG = (\bX, \bE)$ given a set $\bZ \in \bX \setminus \{X,Y\}$ if and only if there is no $\bZ$-active path in $\cG$ between $X$ and $Y$; a $\bZ$-active path is a simple path $Q$ such that any vertex from $\bZ$ on $Q$ occurs as a collider and any vertex from $\bX \setminus \bZ$ appears as a non-collider.
Two nodes are \emph{d-connected} if they are not d-separated.
It is known that $X$ is d-separated from its non-descendants given its parents \cite[Section 3.3.1, Corollary 4]{pearl1988probabilistic}.

A probability distribution $\bbP$ is said to be \emph{Markov with respect to a DAG $\cG$} if d-separation in $\cG$ implies conditional independence in $\bbP$.
Note that any distribution is Markov with respect to the complete DAG, since there are no d-separations implied by this kind of DAG.
(Moreover, any Bayes net over a complete DAG requires $2^{|\bX|} - 1$ parameters to describe.)
A probability distribution $\bbP$ is said to be \emph{Markov with respect to a Bayes net $\cG$} if $\bbP$ is Markov with respect to the underlying DAG of $\cG$.

\subsection{Some Problems of Interest}

\cite{chickering2004large} introduced a decision problem about learning Bayes nets from data, called \LEARN, and proved that \LEARN\ is $\NP$-hard by showing a reduction from the $\NP$-hard decision problem degree-bounded feedback arc set (\DBFAS).
See \Cref{def:DBFAS} and \Cref{def:LEARN}.

\begin{definition}[The \DBFAS\ decision problem]
\label{def:DBFAS}
Given a directed graph $\cG = (\bX, \bE)$ with maximum vertex degree of $3$, and a positive integer $k \leq |\bE|$, determine whether there is a subset of edges $\bE' \subseteq \bE$ with of size $|\bE'| \leq k$ such that $\bE'$ contains at least one directed edge from every directed cycle in $\cG$.
\end{definition}

\begin{definition}[The \LEARN\ decision problem]
\label{def:LEARN}
Given variables $\bX = (X_1, \ldots, X_n)$, a probability distribution $\bbP$ over $\bX$, and a parameter bound $p \in \N$, determine whether there exists a Bayes net $\cG$ with at most $p$ parameters such that $\bbP$ is Markov with respect to $\cG$.
\end{definition}

In our work, we focus on the particular \LEARN\ instances $(\bX, \bbP, p)$ used in \cite{chickering2004large}, namely \LEARNDBFAS.

\begin{definition}[The \LEARNDBFAS\ decision problem]
Let $R$ denote the reduction of \cite{chickering2004large} from \DBFAS\ to \LEARN.
We define as \LEARNDBFAS\ the set of instances of \LEARN\ that are in the range of $R$.
\end{definition}

That is, for any instance $I_L$ of \LEARNDBFAS, there is some instance $I_D$ of \DBFAS\ such that $R\cprn{I_D} = I_L$.

An \emph{independence oracle for a distribution $\bbP$} is an oracle that can determine, in constant time, whether or not $U \ind V \mid \bZ$ for any $U, V \in \bX$ and $\bZ \subseteq \bX \setminus \{U,V\}$.
We will use the following result by \cite{chickering2004large}.

\begin{theorem}[\cite{chickering2004large}]
\label{thm:LEARNDBFAS-NP-hard}
\LEARNDBFAS\ is $\NP$-hard even when one is given access to an independence oracle.
\end{theorem}

\subsection{Selecting a Close Distribution With Finite Samples}

The classic method to select an approximate distribution amongst a set of candidate distributions is via the Scheff\'{e} tournament of \cite{devroye2001combinatorial}, which provides a logarithmic dependency on the number of candidates.

In our work, we will use the Scheff\'{e}-based algorithm of \cite{daskalakis2014faster},%
\footnote{Their result is actually more general than what we stated here.
For instance, they only require sample access to the distributions in $\Qset = \{\bbQ_1, \ldots, \bbQ_m\}$ while our setting is simpler as we have explicit descriptions of each of these distributions.}
which given sample access to an input distribution and explicit access to some candidate distributions, outputs with high probability a candidate distribution that is sufficiently close, in total variation (TV) distance ($\dtv$), to the input distribution.

\begin{theorem}[\cite{daskalakis2014faster}]
\label{thm:FastTournament}
Fix any accuracy parameter $\eps > 0$ and confidence parameter $\delta > 0$.
Suppose there is a distribution $\bbP$ over variables $\bX$ and a collection of explicit distributions $\Qset = \{\bbQ_1, \ldots, \bbQ_m\}$, where each distribution $\bbQ_i$ is defined over the same set $\bX$ and there exists some $\bbQ^* \in \Qset$ such that $\dtv(\bbP, \bbQ) \leq \eps$.
Then, there is an algorithm that uses $\cO \left( \frac{\log 1/\delta}{\eps^2} \log m \right)$ samples from $\bbP$ and returns some $\bbQ \in \Qset$ such that $\dtv(\bbP, \bbQ) \leq 10 \eps$ with success probability at least $1 - \delta$ and running time $\poly(m, 1/\delta, 1/\eps^2)$.
\end{theorem}

To curate a set of candidates $\Qset$, we rely on the following lemma of \cite{brustle2020multi} which states that any distribution $\bbQ$ which approximately agrees with $\bbP$ on the \emph{local} conditional distribution at each node will be close in TV distance to $\bbP$ on the entire domain.

\begin{lemma}[\protect\cite{brustle2020multi}]
\label{lem:conditional-is-additive}
Suppose $\bbP$ and $\bbQ$ are Bayes nets on the same DAG $\cG = (\bX, \bE)$ with $n$ nodes.
If
\[
\tv \Big( \bbP(X \mid \Pa_{\cG}(X) = \sigma), \bbQ(X \mid \Pa_{\cG}(X)) = \sigma \Big)
\leq \frac{\eps}{n}
\]
for all nodes $X \in \bX$ and possible parent values $\sigma \in \Sigma^{|\Pa_{\cG}(X)|}$, then $\tv(\bbP, \bbQ) \leq \eps$.
\end{lemma}

Although there are infinitely many possible distributions, since we are satisfied with an approximately close distribution, one can discretize the space via an $\eps$-net.

\begin{definition}[$\eps$-nets; \protect\cite{vershynin2018high}]
Fix a metric space $(\bm{T},d)$.
For any subset $\bm{K} \subseteq \bm{T}$ and $\eps > 0$, a subset $\bm{N} \subseteq \bm{K}$ is called an $\eps$-net of $\bm{K}$ if every point in $\bm{K}$ is within distance $\eps$ to some point in $\bm{N}$.
That is, $\forall x \in \bm{K}, \exists x_0 \in \bm{N}$ such that $d(x, x_0) \leq \eps$.
We say that $\bm{N}$ $\eps$-covers $\bm{K}$.
\end{definition}

As we shall see in \Cref{sec:upper-bound}, the candidate set $\Qset$ will be created by computing an $\frac{\eps}{n}$-net with respect to the TV distance and then applying \cref{lem:conditional-is-additive} suitably.

\subsection{Other Related Work}

We have already referred to some papers that are relevant to our work.
We resume this discussion here.
\cite{dasgupta1999learning} considers the task of learning the maximum-likelihood polytree from data.
The main result of this paper is that the optimal branching (or Chow-Liu tree) is a good approximation to the best polytree.
This result is then complemented by the observation that this learning problem is $\NP$-hard, even to approximately solve within some constant factor.

\cite{teyssier2005ordering} propose a simple heuristic method for addressing the task of learning Bayes nets.
Their approach is based on the fact that the best network (of bounded in-degree) consistent with a given node ordering can be found efficiently.

\cite{elidan2008learning} present a method for learning Bayes nets of bounded treewidth that employs global structure modifications and that is polynomial both in the size of the graph and the treewidth bound.
At the heart of their method is a dynamic triangulation, that they update in a way which facilitates the addition of chain structures that increase the bound on the model’s treewidth by at most one.

\cite{friedman2013learning} introduce an algorithm that achieves learning by restricting the search space.
Their iterative algorithm restricts the parents of each variable to belong to a small subset of candidates.
They then search for a network that satisfies these constraints and the learned network is then used for selecting better candidates for the next iteration.

\cite{ganian2021theComplexity} investigate the parameterized complexity of Bayesian Network Structure Learning (BNSL).
They show that parameterizing BNSL by the size of a feedback edge set yields fixed-parameter tractability.

\cite{kuipers2022efficient} combine constraint-based methods with greedy or Markov chain Monte Carlo (MCMC) schemes in a method which reduces the complexity of MCMC approaches to that of a constraint-based method.

\section{Technical Overview}

\label{sec:technical-overview}

Here, we give a brief high-level overview of the techniques used in our results of \cref{thm:REALIZABLE-LEARN-is-NP-hard} and \cref{thm:parameter-sample-complexity}.

\subsection{\texorpdfstring{$\NP$}{NP}-Hardness of the Realizable Case}

By \Cref{thm:LEARNDBFAS-NP-hard}, it would suffice to prove that the existence of a polynomial time algorithm for \PROMISE\ implies that \LEARNDBFAS\ instances can be solved in polynomial time if one has access to an independence oracle.
The desired result will then follow from the facts that we can efficiently $(a)$ compute the number of parameters of a Bayes net and $(b)$ decide whether a given distribution is Markov with respect to a given Bayes net (when given access to an independence oracle).

Suppose we have a polynomial time algorithm \texttt{Learner} for \PROMISE.
\label{ftnt:promise}
Note that it is conventional to assume that such an algorithm \emph{always} halts within some polynomial-time bound, and outputs \emph{some} Bayes net, \emph{even when the respective promise is violated}.
We define and analyze the following reduction:
\begin{quote}
Given an arbitrary instance $(\bm{X}, \bbP, p)$ of \PROMISE, run \texttt{Learner} to obtain a Bayes net $\mathcal{G}$.
Then check whether $\mathcal{G}$ has at most $p$ parameters and (while using an independence oracle) check whether or not $\bbP$ is Markov with respect to $\mathcal{G}$.
If \emph{both} of these checks are positive, then output YES.
Otherwise, output NO.
See \cref{sec:np-hard} for the formal proof of \cref{thm:REALIZABLE-LEARN-is-NP-hard}.
\end{quote}

\subsection{Approximately Learning Parameter-Bounded Bayes Networks}

The main idea is to construct an $\eps$-net over all possible DAGs that satisfy the parameter upper bound $p$, and then apply a well-known bound from the density estimation literature.

For this purpose, we need to count all possible Bayes nets that satisfy the parameter upper bound $p$.
By a counting argument, we see that there are not many possible DAGs that give rise to some Bayes net of at most $p$ parameters.
Then, by a counting argument again, we see that there are only a few conditional distributions that are Markov with respect to a Bayes net $\cG$ over a DAG that realizes a given in-degree sequence.
Thus we are able to bound the number of distributions that cover all possible conditional distributions which are Markov with respect to $\mathcal{G}$.
See \cref{sec:upper-bound} for the formal proof of \cref{thm:parameter-sample-complexity}.

\section{\PROMISE\ is \texorpdfstring{$\NP$}{NP}-hard}

\label{sec:np-hard}

To show that \PROMISE\ is hard, we reduce \LEARNDBFAS\ to \PROMISE\ by making polynomially-many calls to an independence oracle.
Given any polynomial time algorithm \texttt{Learner} that solves \PROMISE, we will forward the \LEARNDBFAS\ instance to \texttt{Learner} and examine the produced Bayes net $\cG$.
We will describe a polynomial time procedure \texttt{Reduction} that uses an independence oracle to determine whether we should correctly output YES or NO for the given \LEARNDBFAS\ instance.
See \cref{fig:np-hardness-reduction} for a pictorial illustration of our reduction strategy.

\begin{figure}[ht]
\centering
\begin{tikzpicture}[scale = 1.2]
\node[draw, rounded corners, align=center] (DBFAS) at (0,0) {\DBFAS\\ instance};
\node[draw, rounded corners, align=center] (LEARN-DBFAS) at (0,-2) 
{\LEARNDBFAS\\ instance};
\node[draw, rounded corners, align=center] (PROMISE) at (0,-4) 
{\PROMISE\\ instance};
\node[align=center] (LEARNER) at (5,-4) {\texttt{Learner}};
\node[align=center] (REDUCTION) at (5,-2) {\texttt{Reduction}};
\draw[-stealth] ($(DBFAS.south) + (-0.2,0)$) -- node[left, midway]{$\bX, \bbP, p$} ($(LEARN-DBFAS.north) + (-0.2,0)$);
\draw[-stealth] ($(LEARN-DBFAS.north) + (0.2,0)$) -- node[right, midway]{YES or NO} ($(DBFAS.south) + (0.2,0)$);
\draw[-stealth] (LEARN-DBFAS) -- node[left, midway, align=center]{$\bX, \bbP, p$}node[right, midway, align=center]{Additional\\ promise} (PROMISE);
\draw[-stealth] (PROMISE) -- node[above, midway]{$\bX, \bbP, p$}node[below, midway]{Promise} (LEARNER);
\draw[-stealth] (LEARNER) -- node[left, midway, align=center]{Bayes\\ net $\cG$} (REDUCTION);
\draw[-stealth] (REDUCTION) -- node[above, midway]{YES or NO} (LEARN-DBFAS);
\end{tikzpicture}
\caption{\protect\cite{Gav77} showed that \DBFAS\ is $\NP$-hard and \protect\cite{chickering2004large} showed that \LEARNDBFAS\ is $\NP$-hard, even when  given access to an independence oracle for $\bbP$.
\PROMISE\ is a variant of \LEARNDBFAS\ with the additional promise that there exists a Bayes net $\cG$ with at most $p$ parameters such that $\bbP$ is Markov with respect to $\cG$.
In this work, we show that \emph{if} one can learn such a Bayes net $\cG$ (via some blackbox polynomial time algorithm \texttt{Learner}), then there is a polynomial time algorithm \texttt{Reduction} that correctly answers \LEARNDBFAS.
Therefore, \PROMISE\ is also $\NP$-hard.}
\label{fig:np-hardness-reduction}
\end{figure}

We begin by observing that one can easily check the number of parameters of Bayes net given its full description.

\begin{restatable}{lemma}{computeparameters}
\label{lem:compute-parameters}
Given a Bayes net over $\cG = (\bX, \bE)$, one can compute the number of its parameters in polynomial time.
\end{restatable}

\begin{proof}
Let $\Sigma_X$ denote the alphabet set of $X \in \bX$.
Then, the number of parameters of $\cG$ is
\[
\sum_{X \in \bX} \left( (|\Sigma_X| - 1) \prod_{U \in \Pa_{\cG}(X)} |\Sigma_U| \right),
\]
which can be computed in polynomial time.
\end{proof}

The following notion of important edges will come handy in the sequel.

\begin{definition}[Important edges]
\label{def:important}
Let $\cG = (\bV, \bE)$ be a DAG and $\bbP$ be a distribution over $\bV$.
Then, an edge $e \in \bE$ is called \emph{$(\cG, \bbP)$-important} if $\bbP$ is Markov with respect to $\cG$ but is not Markov with respect to $\cG' = (\bV, \bE \setminus \{e\})$.
\end{definition}

To check whether $\bbP$ is Markov with respect to $\cG$, one could verify that any d-separation in $\cG$ implies conditional independence in $\bbP$.
However, this computation seems to be intractable.
In contrast, \cref{cor:check-inclusion-using-an-independence-oracle-general-case} gives a polynomial time algorithm that checks this \emph{while using an independence oracle}.

The correctness of \cref{cor:check-inclusion-using-an-independence-oracle-general-case} follows from \cref{lem:importance-via-oracle} and \cref{lem:a-subgraph-defines-a-distribution}.

\begin{restatable}{lemma}{importanceviaoracle}
\label{lem:importance-via-oracle}
Suppose $\bbP$ on variables $\bX$ is Markov with respect to $\cG = (\bX, \bE)$.
Then, an edge $A \to B$ in $\bE$ is not $(\cG, \bbP)$-important if $A \ind B \mid \Pa_{\cG}(B) \setminus \{A\}$.
\end{restatable}

\begin{proof}
Consider an arbitrary edge $A \to B$ in $\bE$ such that $A \ind B \mid \Pa_{\cG}(B) \setminus \{A\}$.
Say, $A = X_j$ and $B = X_k$.
Letting $\cG' = (\bV, \bE \setminus (A,B))$ be a subgraph of $\cG$ that does not contain the edge $A \to B$, we see that
\begin{align*}
\bbP(\bx)
= &\; \prod_{i=1}^n \bbP(x_i \mid \pa_{\cG}(X_i)) && (\ast)\\
= &\; \bbP(x_k \mid \pa_{\cG}(X_k)) \cdot \prod_{i \in [n] \setminus k} \bbP(x_i \mid \pa_{\cG}(X_i))\\
= &\; \bbP(x_k \mid \pa_{\cG}(X_k) \setminus x_j) \cdot \prod_{i \in [n] \setminus k} \bbP(x_i \mid \pa_{\cG}(X_i)) && (\dag)\\
= &\; \bbP(x_k \mid \pa_{\cG'}(X_k)) \cdot \prod_{i \in [n] \setminus k} \bbP(x_i \mid \pa_{\cG'}(X_i)) && (\ddag)\\
= &\; \prod_{i=1}^n \bbP(x_i \mid \pa_{\cG'}(X_i)),
\end{align*}
where $(\ast)$ is due to $\bbP$ being Markov with respect to $\cG$, $(\dag)$ is due to $X_j \ind X_k \mid \Pa_{\cG}(X_k) \setminus \{X_j\}$, and $(\dag)$ is due to the definition of $\cG'$.
Since $\bbP(\bx) = \prod_{i=1}^n \bbP(x_i \mid \pa_{\cG'}(X_i))$, we see that $\bbP$ is also Markov with respect to $\cG'$, and so the edge $A \to B$ is not $(\cG, \bbP)$-important.
\end{proof}

\begin{restatable}{lemma}{asubgraphdefinesadistribution}
\label{lem:a-subgraph-defines-a-distribution}
Suppose a distribution $\bbP$ on variables $\bX$ is Markov with respect to a DAG $\cG = (\bX, \bE)$.
Let $\cG' = (\bX, \bE')$ be an edge-induced DAG of $\cG$ with $\bE' \subseteq \bE$.
Then, $\bbP$ is Markov with respect to $\cG'$ if and only if $A \ind B \mid \Pa_{\cG}(B) \setminus \{A\}$ for all edges $A \to B$ in $\bE \setminus \bE'$.
\end{restatable}

\begin{proof}
We prove each direction separately.

\paragraph{($\Leftarrow$)}

Suppose that $\bbP$ is Markov with respect to $\cG'$.
Consider an arbitrary edge $A \to B \in \bE \setminus \bE'$.
Since $A$ is an ancestor of $B$ in $\cG'$, we have that $A$ remains a non-descendant of $B$ in $\cG'$ after removing the edge $A \to B$.
So, $A$ and $B$ are d-separated in $\cG'$ given $\Pa_{\cG'}(B) \setminus \{A\}$, and so $A \ind B \mid \Pa_{\cG'}(B)$ by the Markov property.
That is, $A \ind B \mid \Pa_{\cG}(B) \setminus \{A\}$.

\paragraph{($\Rightarrow$)}

Suppose that $A \ind B \mid \Pa_{\cG}(B) \setminus \{A\}$ for all edges $A \to B$ in $\bE \setminus \bE'$.
Order the edges in $\bE \setminus \bE'$ in an arbitrary sequence, say $e_1, \ldots, e_{|\bE \setminus \bE'|}$.
Let us remove these edges sequentially, resulting in a sequence of edge-induced DAGs $\cG = \cG_0, \cG_1, \ldots, \cG_{|\bE \setminus \bE'|} = \cG'$, where $\cG_i$ is the edge-induced DAG obtained from removing edges $\{e_1, \ldots, e_i\}$ from $\cG$.
Observe that non-descendant relationships are preserved as we remove edges, i.e., if $A$ is a non-descendant of $B$ in $\cG$, then it is also a non-descendant of $B$ in $\cG_i$ for any $i \in \{1, \ldots, |\bE \setminus \bE'|\}$.
So, we can apply \cref{lem:importance-via-oracle} repeatedly:
For any $i \in \{1, \ldots, |\bE \setminus \bE'|\}$, the edge $e_i$ is not $(\cG_{i-1}, \bbP)$-important, so $\bbP$ is Markov with respect to $\cG_i$.
That is, $\bbP$ is Markov with respect to $\cG'$.
\end{proof}

\begin{restatable}{corollary}{checkinclusion}
\label{cor:check-inclusion-using-an-independence-oracle-general-case}
Suppose $\bbP$ is a distribution over $\bX$ and $\cG$ is a Bayes net over the same set of variables $\bX$.
Then, there is a polynomial time algorithm that uses an independence oracle for $\bbP$ to decide whether or not $\bbP$ is Markov with respect to $\cG$.
\end{restatable}

\begin{proof}
Consider the following algorithm $\texttt{Checker}$:
\begin{quote}
Given Bayes net $\cG$ over a DAG $(\bX, \bE')$, consider a DAG $\cK = (\bX, \bE)$, with $\bE' \subseteq \bE$, that is a complete supergraph of $(\bX, \bE')$.
If \emph{every} edge $A \to B \in \bE \setminus \bE'$ satisfies $A \ind B \mid \Pa_{\cG}(B) \setminus \{A\}$, output YES.
Otherwise, output NO.
\end{quote}
Note that since any distribution is Markov with respect to the complete DAG (see \Cref{sec:prel-graphs}), $\bbP$ is Markov with respect to $\cK$.

The correctness of $\texttt{Checker}$ follows from \cref{lem:a-subgraph-defines-a-distribution}.
$\texttt{Checker}$ runs in polynomial time as $\cK$ can be created in polynomial time with respect to the size of $\bX$, and the number of edges in $\bE \subseteq \bE'$ to check is polynomial in the size of $\bX$.
\end{proof}

We are now ready to formally prove our first main result.

\nphardthm*

\begin{proof}
It suffices to show the existence of a polynomial time algorithm for \PROMISE\ implies that \LEARNDBFAS\ instances can be answered in polynomial time with access to an independence oracle; see \cref{fig:np-hardness-reduction}.

Suppose we have a polynomial time algorithm \texttt{Learner} for \PROMISE.
Let us define a reduction algorithm \texttt{Reduction} as follows:
\begin{quote}
Given an instance $(\bX, \bbP, p)$ of \LEARNDBFAS, run \texttt{Learner} to obtain a Bayes net $\cG$ (see \Cref{ftnt:promise}; this is a natural assumption for algorithms solving a search promise problem).
Compute the number of parameters of $\cG$.
Run algorithm $\texttt{Checker}$ of \cref{cor:check-inclusion-using-an-independence-oracle-general-case} on $\cG$ to check whether or not $\bbP$ is Markov with respect to $\cG$.
Output YES if $\cG$ has at most $p$ parameters \emph{and} $\bbP$ is Markov with respect to $\cG$; else, output NO.
\end{quote}
The correctness of \texttt{Reduction} follows from the assumption that \texttt{Learner} produces a Bayes net $\cG$ with at most $p$ parameters such that $\bbP$ is Markov with respect to $\cG$ (if the underlying promise is satisfied; otherwise, the output is an arbitrary Bayes net), and the correctness of $\texttt{Checker}$.
By assumption, \texttt{Learner} is a polynomial time algorithm.
By \cref{lem:compute-parameters}, we can compute the number of parameters of $\cG$ in polynomial time.
By \cref{cor:check-inclusion-using-an-independence-oracle-general-case}, \texttt{Checker} is also a polynomial time algorithm.
Therefore, the overall running time for \texttt{Reduction} is polynomial.
\end{proof}

\section{Approximating Bayes Nets}

\label{sec:upper-bound}

Our strategy for proving our finite sample complexity result (\cref{thm:parameter-sample-complexity}) follows that of \cite[Theorem 10]{brustle2020multi}, but we specialize the analysis to the setting where we are given a parameter bound instead of a degree bound.
As discussed in \cref{sec:intro}, our result is a generalization of their result since an upper bound on the in-degrees implies a (possibly loose) parameter upper bound.

\subsection{Some Graph Counting Arguments}

To prove \cref{thm:parameter-sample-complexity}, we require an upper bound on the number of possible Bayes nets on $n$ nodes that have at most $p$ parameters (\cref{lem:parameter-graph-count}).
To obtain such a result, we first relate the number of parameters $p$ with a specific given in-degree sequence $(d_1^-, \ldots, d_n^-)$ of a Bayes net, then we upper bound the total number of Bayes nets that has at most $p$ parameters by summing over all suitable in-degree sequences $\bd^- = (d_1^-, \ldots, d_n^-)$.

Consider an arbitrary Bayes net $\cG$ with in-degree sequence $(d_1^-, \ldots, d_n^-)$ and each node taking on $|\Sigma|$ values.
Since the conditional distribution for vertex $X_i$ is fully described when we know $\bbP(x_i \mid \pa_{\cG}(X_i))$ for $|\Sigma|-1$ possible values of $x_i$, with respect to $\abs{\Sigma}^{d_i^-}$ possible values of $\pa_{\cG}(X_i)$.
Therefore, we see that the Bayes net has
\[
\sum_{i=1}^n \left( \left( |\Sigma| - 1 \right) |\Sigma|^{d_i^-} \right)
= \left( |\Sigma| - 1 \right) \left( \sum_{i=1}^n |\Sigma|^{d_i^-} \right)
\]
parameters.
Note that this is the exact same reasoning as in \cref{lem:compute-parameters}.
So, if the Bayes net has at most $p$ parameters, then
\begin{equation}
\label{eq:rearranged-p-bound}
\sum_{i=1}^n |\Sigma|^{d_i^-}
= |\Sigma|^{d_1^-} + \ldots + |\Sigma|^{d_n^-}
\leq \frac{p}{|\Sigma| - 1}.
\end{equation}
By the AM-GM inequality, we have that
\begin{equation}
\label{eq:am-gm-bound}
\sum_{i=1}^n |\Sigma|^{d_i^-}
\geq n \left(\prod_{i=1}^n |\Sigma|^{d_i^-}\right)^{\frac{1}{n}}
= n |\Sigma|^{\frac{1}{n} \sum_{i=1}^n d_i^-}.
\end{equation}
Combining \cref{eq:rearranged-p-bound,eq:am-gm-bound} together gives us
\begin{equation}
\label{eq:indeg-sum-bound}
d_1^- + \ldots + d_n^-
\leq n \frac{\log \left( \frac{p}{n \left( |\Sigma| -1 \right)} \right) }{\log |\Sigma|}.
\end{equation}
The following lemma is a combinatorial fact upper bounding on the number of graphs that realize a given degree sequence, which may be of independent interest beyond being used to prove \cref{lem:parameter-graph-count}.

\begin{restatable}{lemma}{DAGincountbound}
\label{lem:DAG-incount-bound}
Given an in-degree sequence $\bd^- = (d_1^-, \ldots, d_n^-)$ with non-negative integers $d_1^-, \ldots, d_n^-$, there are at most $\prod_{i=1}^n \binom{n-1}{d_i^-}$ DAGs that realize $\bd^-$.
\end{restatable}

\begin{proof}
Fix an arbitrary labelling of vertices from $X_1$ to $X_n$ and consider the sequential process of adding edges into $X_1, \ldots, X_n$.
For $X_1$, there are $\binom{n-1}{d_1^-}$ ways to add $d_1^-$ incoming edges that end at $X_1$.
For $X_2$, there are $\binom{n-1}{d_2^-}$ possibilities.
For $X_3$, there are \emph{at most} $\binom{n-1}{d_3^-}$ possibilities.
Note that some of these choices would be \emph{incompatible} with earlier edge choices as the newly added edges may cause directed cycles to be formed.
We repeat this edge adding process until all vertices have added their incoming edges to the graph.
So, the \emph{upper bound} is $\prod_{i=1}^n \binom{n-1}{d_i^-}$.
\end{proof}

\begin{lemma}
\label{lem:parameter-graph-count}
Suppose that every node takes on at most $|\Sigma|$ values.
Then, there are at most
\[
\left( n-1 \right)^{\frac{n \log \left( \frac{p}{n \left( |\Sigma| -1 \right)} \right)}{\log \abs{\Sigma}}}
e^n \left(\frac{\log \left( \frac{p}{n \left( |\Sigma| -1 \right)} \right)}{\log \abs{\Sigma}} + 1\right)^n
\]
possible DAGs over $n$ nodes that may be used to define some Bayes net that has at most $p$ parameters.
\end{lemma}

\begin{proof}
By \cref{lem:DAG-incount-bound}, there are $\prod_{i=1}^n \binom{n-1}{d_i^-}$ possible DAGs realizing any fixed in-degree sequence $\bd^- = (d_1^-, \ldots, d_n^-)$.
Let $(\ast)$ denote the condition that an in-degree sequence $\bd^-$ yields a graph that has at most $p$ parameters.
Then,
\begin{align*}
\sum_{\text{$\bd^-$ satisfies $(\ast)$}} \prod_{i=1}^n \binom{n-1}{d_i^-}
\leq &\; \sum_{\text{$\bd^-$ satisfies $(\ast)$}} \left(n-1\right)^{d_1^- + \ldots + d_n^-}\\
\leq &\; \left(n-1\right)^{n \frac{\log\left(\frac{p}{n \left(|\Sigma|-1\right)}\right)}{\log |\Sigma|}} \sum_{\text{$\bd^-$ satisfies $(\ast)$}} 1\\
\leq &\; (n-1)^{n \frac{\log\left(\frac{p}{n \left( |\Sigma|-1 \right)}\right)}{\log |\Sigma|}} \binom{n \frac{\log \left( \frac{p}{n \left( |\Sigma| -1 \right)} \right)}{\log |\Sigma|} + n}{n}\\
\leq &\; \left( n-1 \right)^{n \frac{\log \left( \frac{p}{n \left( |\Sigma| - 1 \right)}\right)}{\log |\Sigma|}} \left(e \left( \frac{\log \left( \frac{p}{n \left( |\Sigma| -1 \right)} \right)}{\log |\Sigma|} + 1 \right) \right)^n
\end{align*}
where the first and last inequalities is because $\binom{n}{k} \leq \left(\frac{en}{k}\right)^k \leq n^k$, the second inequality is due to \cref{eq:indeg-sum-bound}, and the third inequality is obtained via standard ``stars and bars'' counting.
That is, we introduce an auxiliary variable $d_0^-$ and count the number of non-negative integer solutions of
\[
d_0^- + d_1^- + \ldots + d_n^-
= n \frac{\log \left( \frac{p}{n \left( |\Sigma| -1 \right)} \right)}{\log |\Sigma|}.
\qedhere
\]
\end{proof}

\subsection{Proof of \texorpdfstring{\Cref{thm:parameter-sample-complexity}}{}}

We are now ready to prove \cref{thm:parameter-sample-complexity}.
Since \cref{lem:conditional-is-additive} tells us that it suffices to approximate each local conditional distribution at each node well.
So, we will consider an $\frac{\eps}{n}$-net over all such distributions and then apply a tournament style argument (\cref{thm:FastTournament}) to pick a good candidate amongst the joint distribution obtained by a combination of such candidate local distributions.

\samplecomplexitythm*

\begin{proof}
Fix a DAG $\cG$ satisfying an arbitrary in-degree sequence $\bd^- = \left(d_1^-, \ldots, d_n^-\right)$.
Then, there are $|\Sigma|^{d_1^-} + \ldots + |\Sigma|^{d_n^-}$ local conditional distributions for any Bayes net over $\cG$.
From \cref{eq:rearranged-p-bound} above, we know that
\[
\sum_{i=1}^n |\Sigma|^{d_i^-}
= |\Sigma|^{d_1^-} + \ldots + |\Sigma|^{d_n^-}
\leq \frac{p}{|\Sigma| - 1}.
\]
Now, consider an arbitrary local distribution over $k = |\Sigma|$ values and let us upper bound the number of points in an $\frac{\eps}{n}$-net for this metric space.
Observe that each possible distribution is essentially an element of the probability simplex $\Delta_k$.
To get an $\frac{\eps}{n}$-net of $\Delta_k$, we discretize vectors by rounding them to their nearest multiple of $\frac{\varepsilon}{n\abs{\Sigma}}$.
If $\pi$ is a probability vector, and $r_\pi$ is its rounding, then $\bars{\pi-r_\pi}_1 \leq \frac{\varepsilon}{n\abs{\Sigma}} \abs{\Sigma} = \frac{\varepsilon}{n}$.
Therefore the number of discretized vectors is at most $\cO \left( \left({n \abs{\Sigma}}/{\eps} \right)^{|\Sigma|} \right)$.

Therefore, for any fixed DAG $\cG$, there is a set of
\begin{equation}
\label{eq:eps-net-count}
m_1 \in \cO\left( \left({n \abs{\Sigma}}/{\eps} \right)^{\frac{p |\Sigma|}{|\Sigma| - 1}} \right)
\end{equation}
distributions that $\frac{\eps}{n}$-cover any possible joint distributions that can be Markov with respect to a Bayes net over $\cG$.
Meanwhile, by \cref{lem:parameter-graph-count}, there are at most
\begin{align}
\label{eq:param-graph-count}
m_2
:= \left( n-1 \right)^{\frac{n \log \left( \frac{p}{n \left( |\Sigma| - 1 \right)} \right)}{\log \abs{\Sigma}}}
e^n \left(\frac{\log \left( \frac{p}{n \left( |\Sigma| -1 \right)} \right)}{\log \abs{\Sigma}} + 1\right)^n
\end{align}
possible DAGs that may be used to define a Bayes net on $n$ nodes that has at most $p$ parameters.

We can now define a set of distributions $\Qset$ over $n$ variables such that there exists $\bbQ^* \in \Qset$ such that $\tv(\bbP, \bbQ) \leq \eps$.
Let us denote $m = |\Qset|$.
Putting together the above bounds, we see that there are at most $m = m_1 \cdot m_2$ candidates suffice, where $m_1$ and $m_2$ are from \cref{eq:eps-net-count,eq:param-graph-count}.
Therefore, with
\begin{align*}
\cO\!\left( \frac{\log \frac{1}{\delta}}{\eps^2} \log m \right)
\subseteq
\cO\!\left(  \frac{\log \frac{1}{\delta}}{\eps^2}  \left(  p \log \left( \frac{n \abs{\Sigma}}{\eps} \right) + \frac{n \log \left( \frac{p}{n \left( |\Sigma| -1 \right)} \right)}{\log \abs{\Sigma}} \log n  \right) \right)
\end{align*}
samples from $\bbP$, \cref{thm:FastTournament} chooses a distribution $\bbQ$ amongst the $m$ candidates such that $\tv(\bbP, \bbQ) \leq \eps$ with success probability at least $1 - \delta$.
\end{proof}

\section{Conclusion}

\label{sec:discussion}

In this work, we showed the hardness result of finding a parameter-bounded Bayes net that represents some distribution $\bbP$, given sample access to $\bbP$, \emph{even under the promise that such a Bayes net exists}.
On a positive note, we gave a finite sample complexity bound sufficient to produce a Bayes net representing a probability distribution $\bbQ$ that is close in TV distance to $\bbP$.
Our results generalize earlier known results of \cite{chickering2004large} and \cite{brustle2020multi} respectively.

An intriguing open question is as follows:
\begin{quote}
Suppose we are given sample access to a distribution $\bbP$ and are promised that there exists a Bayes net on $\cG$ with at most $p$ parameters such that $\bbP$ is Markov with respect to $\cG$.
Is it hard to find a Bayes net $\cG'$ that has $\alpha \cdot p$ parameters such that $\bbP$ is Markov with respect to $\cG'$ (where $\cG'$ may not be $\cG$), for some constant $\alpha > 1$?
\end{quote}
Note that the hardness construction of \cite{chickering2004large} only displayed an \emph{additive} gap in the parameter bound.
We conjecture that it is also hard to obtain such a \emph{multiplicative} gap in the parameter bound, even in the promise setting.

\section*{Acknowledgements}

We would like to thank Dimitris Zoros for helpful discussions.
This research/project is supported by the National Research Foundation, Singapore under its AI Singapore Programme (AISG Award No: AISG-PhD/2021-08-013).
The work of AB was supported in part by National Research Foundation Singapore under its NRF Fellowship Programme (NRF-NRFFAI-2019-0002) and an Amazon Faculty Research Award.
SG’s work was partially supported by the SERB CRG Award CRG/2022/007985 and an IIT Kanpur initiation grant.
Part of this work was done while the authors were visiting the Simons Institute for the Theory of Computing.


\end{document}